\documentclass[graybox]{svmult}

\usepackage{type1cm}        
\usepackage{graphicx}        
\usepackage[bottom]{footmisc}

\usepackage{natbib}         
\usepackage{amssymb}
\usepackage{amsmath}
\usepackage{bm}
\usepackage{mathtools}
\usepackage{xcolor}
\usepackage{tikz}
\usepackage[shortlabels]{enumitem}

\usepackage[hidelinks]{hyperref}

\DeclarePairedDelimiter{\group}{(}{)}

\DeclarePairedDelimiter{\set}{\{}{\}}

\DeclarePairedDelimiter{\abs}{\vert}{\vert}

\newcommand{\naturals}{\mathbb{N}}

\newcommand{\reals}{\mathbb{R}}
\newcommand{\posreals}{\reals_{>0}}
\newcommand{\nonnegreals}{\reals_{\geq0}}

\newcommand{\states}{\mathcal{X}}

\newcommand{\gbls}{\mathcal{L}}

\newcommand{\gblson}[1]{\gbls(#1)}
\newcommand{\gblsonstates}{\gblson{\states}}

\newcommand{\optgt}[1][]{\succ_{#1}}

\newcommand{\opt}[1][]{u_{#1}}
\newcommand{\altopt}[1][]{v_{#1}}

\newcommand{\opts}{\mathcal{V}} 
\newcommand{\posopts}{\opts_{\optgt0}}

\newcommand{\optset}[1][]{A_{#1}}
\newcommand{\altoptset}[1][]{B_{#1}}

\newcommand{\optsets}{\mathcal{Q}}

\newcommand{\assessment}{\mathcal{A}}

\newcommand{\desirset}[1][]{D_{#1}}

\newcommand{\desirsets}{\mathbf{D}}

\newcommand{\rejectset}[1][]{K_{#1}}

\newcommand{\setofdesirsets}{\mathcal{D}}
\newcommand{\choicefun}[1][]{C_{#1}}
\newcommand{\rejectfun}[1][]{R_{#1}}

\newcommand{\ddualopts}[1][]{\opts^\circ}

\newcommand{\ldualopts}[1][]{\underline{\opts}^\ast}

\newcommand{\nml}[1][{\opt[o]}]{\normalise_{#1}}
\newcommand{\cset}[3][]{\set[#1]{#2\colon#3}}

\newcommand{\then}{\Rightarrow}

\DeclareMathOperator{\posi}{posi}

\DeclareMathOperator{\normalise}{N}

\usepackage{etoolbox}
\newtoggle{arxiv}
\toggletrue{arxiv}

\begin{document}

\title*{Choice functions based on
sets of strict partial orders:
an axiomatic characterisation}
\titlerunning{Choice functions based on
sets of strict partial orders}
\author{Jasper De Bock}
\institute{Jasper De Bock \at Ghent University, ELIS, Flip, \email{jasper.debock@ugent.be}
}
%
%
\maketitle


\abstract{Methods for choosing from a set of options are often based on a strict partial order on these options, or on a set of such partial orders. I here provide a very general axiomatic characterisation for choice functions of this form. It includes as special cases axiomatic characterisations for choice functions based on (sets of) total orders, (sets of) weak orders, (sets of) coherent lower previsions and (sets of) probability measures.}

\section{Introduction}

A choice function $\choicefun$ is a simple mathematical model for representing choices: for any set of options $\optset$, it returns a subset $\choicefun(\optset)\subseteq\optset$ of options that are ``chosen'' from $\optset$, in the sense that the options in $\rejectfun(\optset)\coloneqq\optset\setminus\choicefun(\optset)$ are rejected and that the options in $\choicefun(\optset)$ are deemed incomparable \citep{seidenfeld2010,
2018vancamp:lexicographic,
pmlr-v103-de-bock19b}.

Such a choice function is often derived from a strict partial order $\succ$ on options, by choosing the options that are maximal---or undominated---with respect to this ordering, or equivalently, by rejecting the options that are dominated:
\begin{equation}\label{eq:Cfromorder}
\choicefun[\succ](\optset)\coloneqq\big\{\opt\in\optset\colon(\nexists\altopt\in\optset)~\altopt\succ\opt\big\}
\end{equation}
More generally, we can also associate a choice function with any set $\mathcal{O}$ of such strict partial orders, by choosing the options in $\optset$ that are maximal with respect to at least one of the considered orderings:
\begin{equation}\label{eq:Cfromorders}
\choicefun[\mathcal{O}](\optset)
\coloneqq\bigcup_{\succ\in\hspace{0.4pt}\raisebox{-0.6pt}{\scriptsize $\mathcal{O}$}}\choicefun[\succ](\optset)
=\{\opt\in\optset\colon(\exists \succ\,\in\raisebox{-0.7pt}{$\mathcal{O}$})\,(\nexists\altopt\in\optset)~\altopt\succ\opt\big\}
\end{equation}
This approach is conservative because $\choicefun[\mathcal{O}]$ will only reject an option if it is rejected---or dominated---with respect to each of the partial orders $\succ$ in $\mathcal{O}$; it can for example be used to represent conservative group decisions, by interpreting every $\succ$ in $\mathcal{O}$ as the preferences of a different group member. 

Choice functions of the form $\choicefun[\succ]$ and $\choicefun[\mathcal{O}]$ appear in various settings, in all sorts of forms and variations, depending on what the options are and which kinds of properties are imposed on the partial orders involved. Maximising expected utility, for example, is a well-known special case of a choice function of the type $\choicefun[\succ]$, where the options are utility functions and one chooses the option(s) whose expected utility with respect to some given probability measure is highest. If the probability measure involved is only know to belong to a set---for example because different group members assign different probabilities---this naturally extends to a choice function of the type $\choicefun[\mathcal{O}]$. 

The first main contribution of this paper is a set of necessary and sufficient conditions for a general choice function $\choicefun$ to be of the form $\choicefun[\succ]$ or $\choicefun[\mathcal{O}]$, for the case where options are elements of a real vector space. More generally, I provide generic necessary and sufficient conditions for the representing orders to satisfy additional properties, provided these properties are expressable in an abstract rule-based form. This leads in particular to representation theorems for choice functions that are based on (sets of) total orders, (sets of) weak orders, (sets of) coherent lower previsions, (sets of) probability measures, and potentially many other types of uncertainty models.
\iftoggle{arxiv}{
	
}{
	Proofs are omitted; they are available in the appendix of an extended online version \citep{smps2020debock:arxiv}.
}

\section{Choice functions based on (sets of) proper orderings}\label{sec:choicefunctionsfromproperorderings}

Let $\opts$ be a real vector space, the elements of which we call options, let $\optsets$ be the set of all---possibly infinite---subsets of $\opts$, excluding the empty set, and let $\optsets_{\emptyset}\coloneqq\optsets\cup\{\emptyset\}$. A choice function $\choicefun\colon\optsets\to\optsets_{\emptyset}$ is a function that, for any option set $\optset\in\optsets$, returns a subset $\choicefun(\optset)$ consisting of the options in $\optset$ that are chosen, or rather, not rejected. 
The elements of $\choicefun(\optset)$ are deemed incomparable when it comes to choosing from $\optset$. We do not exclude the possibility that $\choicefun(\optset)=\emptyset$; this may for example be reasonable if $\optset$ is an infinite option set whose elements are linearly ordered. Imagine choosing the highest natural number, for example. Every natural number is rejected, yet none can reasonably be chosen.


We are particularly interested in choice functions of the form $\choicefun[\succ]$ or $\choicefun[\mathcal{O}]$, corresponding to a strict partial order $\succ$ on $\opts$ or to a set $\mathcal{O}$ of such orders, respectively. 
Besides being strict partial orders (\ref{ax:order:irreflexive} and~\ref{ax:order:transitive} below), we will also require these orders to be compatible with the vector space operations of the real vector space $\opts$ (\ref{ax:order:multiplication} and~\ref{ax:order:addition}). We will call such orders proper.

\begin{definition}
\label{def:proporder}
A binary relation $\succ$ on $\opts$ is a \emph{proper order} if, for all $\opt,\altopt,w\in\opts$ and $\lambda\in\posreals\coloneqq\{\lambda\in\reals\colon\lambda>0\}$, it satisfies the following properties:
\begin{enumerate}[label=$\mathrm{PO}_{\arabic*}$.,ref=$\mathrm{PO}_{\arabic*}$,leftmargin=*]
\item\label{ax:order:irreflexive} $\opt\not\succ\opt$\hfill irreflexivity
\item\label{ax:order:transitive} if $\opt\succ\altopt$ and $\altopt\succ w$, then also $\opt\succ w$\hfill transitivity
\item\label{ax:order:multiplication} if $\opt\succ\altopt$ then also $\lambda\opt\succ\lambda\altopt$
\item\label{ax:order:addition} if $\opt\succ\altopt$ then also $\opt+w\succ\altopt+w$.
\end{enumerate}
We denote the set of all proper orders by $\mathbf{O}$.
\end{definition}

Since \ref{ax:order:addition} implies that $u\succ v$ is equivalent to $u-v\succ0$, every proper order $\succ$ is completely characterised by the set of options
\begin{equation}\label{eq:fromOtoD}
\desirset[\succ]\coloneqq\{\opt\in\opts\colon\opt\succ0\}.
\end{equation}

In fact, as established in Proposition~\ref{prop:orderequivD}, proper orders are in one-to-one correspondence with convex cones $\desirset$ in $\opts$ that are blunt, meaning that they do not include $0$. We call any such $\desirset$ a proper set of options.

\begin{definition}
\label{def:propdesir}
A \emph{set of options} is a---possibly empty---subset $\desirset$ of $\opts$. It is called \emph{proper} if it satisfies the following properties:
\begin{enumerate}[label=$\mathrm{PD}_{\arabic*}$.,ref=$\mathrm{PD}_{\arabic*}$,leftmargin=*]
\item\label{ax:desirs:nozero} $0\notin\desirset$
\item\label{ax:desirs:sum} if $\opt\in\desirset$ and $\altopt\in\desirset$, then $\opt+\altopt\in\desirset$
\item\label{ax:desirs:lambda} if $\opt\in\desirset$ and $\lambda\in\posreals$, then $\lambda\opt\in\desirset$.
\end{enumerate}
We denote the set of all proper sets of options by $\desirsets$.
\end{definition}

\begin{proposition}\label{prop:orderequivD}
A binary relation $\succ$ on $\opts$ is a proper order if and only if there is a proper set of options $\desirset$ such that
\begin{equation}\label{eq:fromDtoO}
\opt\succ\altopt
\Leftrightarrow
\opt-\altopt\in\desirset
\text{ for all $\opt,\altopt\in\opts$.}
\end{equation}
This $\desirset$ is then necessarily unique and equal to $\desirset[\succ]$.
\end{proposition}

It follows that choice functions of the form $\choicefun[\succ]$ are completely determined by a single proper set of options $\desirset$, whereas choice functions of the form $\choicefun[\mathcal{O}]$ are characterised by a set $\setofdesirsets$ of such proper sets of options.

\section{Imposing additional properties through sets of rules}\label{sec:abstractaxiom}

In most cases, rather than consider (sets of) arbitrary proper orders, one whishes to consider particular types of such orders, by imposing additional properties besides \ref{ax:order:irreflexive}--\ref{ax:order:addition}. 
Rather than treat each of these types seperately, we will instead consider an abstract axiom that includes a variety of them as special cases. We will characterise this axiom by means of a set of rules. Each such rule is a pair $(\mathcal{A},B)$, with $\mathcal{A}\subseteq\optsets$ and $B\in\optsets_{\emptyset}$. A set of rules $\mathcal{R}$, therefore, is a subset of $\mathcal{P}(\optsets)\times\optsets_{\emptyset}$, with $\mathcal{P}(\optsets)$ the powerset of $\mathcal{Q}$.

\begin{definition}\label{def:abstractaxiom:D}
Let $\mathcal{R}$ be a set of rules. For any proper set of options $\desirset$, we then say that $\desirset$ is $\mathcal{R}$-compatible if for all $(\mathcal{A},B)\in\mathcal{R}$:
\begin{enumerate}[label=$\mathrm{PD}_{\mathrm{\mathcal{R}}}$.,ref=$\mathrm{PD}_{\mathrm{\mathcal{R}}}$,leftmargin=*]
\item\label{ax:desirsets:abstractaxiom} if $A\cap\desirset\neq\emptyset$ for all $A\in\mathcal{A}$, then also $B\cap\desirset\neq\emptyset$.
\end{enumerate}
A proper order $\succ$ is called $\mathcal{R}$-compatible if $\desirset[\succ]$ is $\mathcal{R}$-compatible. 
We let $\desirsets_{\mathcal{R}}$ and $\mathbf{O}_{\mathcal{R}}$ be the set of all $\mathcal{R}$-compatible proper sets of options and orders, respectively.
\end{definition}

A first important special case are (sets of) rules of the type $(\emptyset,\altoptset)$. Since $\assessment=\emptyset$ makes the premise of \ref{ax:desirsets:abstractaxiom} trivially true, such a rule allows one to express that $\altoptset$ should contain at least one option $\opt$ such that $\opt\succ0$.
If $\opts$ is the set $\gblsonstates$ of all gambles---bounded real functions---on some state space $\states$, coherence of $\desirset[\succ]$ for example corresponds to the set of rules
\begin{equation*}\mathcal{R}_{\mathrm{C}}
\coloneqq\big\{
  (\emptyset,\{\opt\})
  \colon
  \opt\in\posopts
\big\},
\end{equation*}
with $\posopts$ equal to $\{\opt\in\opts\colon\inf\opt>0\}$ or $\{\opt\in\opts\colon\inf\opt\geq0, \opt\neq0\}$, depending on the specific type of coherence~\citep{quaeghebeur2012:itip}. In both cases, this set of rules imposes that $\posopts$ should be a subset of $\desirset[\succ]$. A second example are total orders. These correspond to the set of rules
\begin{equation}\label{eq:abstracttotality}
\mathcal{R}_{\mathrm{T}}
\coloneqq
\big\{
  (\emptyset,\{\opt,-\opt\})
  \colon
  \opt\in\opts\setminus\set{0}
\big\},
\end{equation}
which imposes that $\opt\succ0$ or $-\opt\succ0$ for all $\opt\neq0$. Given the properness of $\succ$, this is equivalent to the totality of $\succ$, meaning that $\opt\succ\altopt$ or $\altopt\succ\opt$ for all $\opt,\altopt\in\opts$ such that $\opt\neq\altopt$.

Another special case are rules of the type $(\assessment,\emptyset)$. Since $\altoptset\cap\desirset\neq\emptyset$ cannot be true for $\altoptset=\emptyset$, rules of this type allow one to express that there should be at least one option set $\optset\in\assessment$ such that $\opt\not\succ0$ for all $\opt\in\optset$. If we for example want to enforce that $\opt[i]\not\succ\altopt[i]$ for all $i$ in some index set $I$, it suffices to let $\mathcal{A}=\{\optset\}$, with $\optset=\{\opt[i]-\altopt[i]\colon i\in I\}$.



Besides total orders, we can also obtain weak orders as a special case of $\mathcal{R}$-compatibility. This can be achieved by imposing any of the following three equivalent sets of rules:
\begin{align*}
\mathcal{R}_{\mathrm{W}}
&\coloneqq
\big\{
	(\{\{u+v,-u-v\}\},\{u,-u,v,-v\})\colon u,v\in\opts
\big\}\\
\mathcal{R}_{\mathrm{W2}}
&\coloneqq
\big\{
  (\{\{u+v\}\},\{u,v\})
  \colon
  u,v\in\opts
\big\}\\
\mathcal{R}_{\mathrm{M}}
&\coloneqq
\big\{
  (\{\posi(\altoptset)\},\altoptset)
  \colon
  \altoptset\in\optsets
  \text{ and }\abs{\altoptset}<+\infty
\big\},
\end{align*}
with
$\posi(\altoptset)\coloneqq
\cset[\big]{\textstyle\sum_{i=1}^n\lambda_i\opt[i]}{n\in\naturals,\lambda_i>0,\opt[i]\in\altoptset}$ the set of finite positive linear combinations of options in $\altoptset$. The first set of rules, $\mathcal{R}_{\mathrm{W}}$, is most closely related to the defining property that makes a strict partial order weak, which is that the incomparability relation---$\opt\not\succ\altopt$ and $\altopt\not\succ\opt$---should be transitive. The second set of rules, $\mathcal{R}_{\mathrm{W}}$, is the simplest yet perhaps least intuitive of the three. The third set of rules, $\mathcal{R}_{\mathrm{M}}$, corresponds to the so-called mixingness of $\desirset[\succ]$~\citep{pmlr-v103-de-bock19b}; it also shows that a strict partial order $\succ$ is weak if and only if $\desirset[\succ]$ is lexicographic~\citep{2018vancamp:lexicographic}, in the sense that its complement $\desirset[\succ]^{\mathrm{c}}$ is a convex cone. The equivalence of these three rules, as well as the fact that they indeed correspond to $\succ$ being a weak order, follows from the properness of $\succ$.

The notion of $\mathcal{R}$-compatibility can also be used to impose several properties at once; it suffices to let $\mathcal{R}$ be the union of the sets of rules that correspond to the individual properties.
 

For example, if $\opts=\gblsonstates$ and $\posopts=\{\opt\in\opts\colon\inf\opt>0\}$, then partial orders that correspond to coherent lower previsions can be obtained by combining coherence with Archimedeanity~\citep{pmlr-v103-de-bock19b}. Coherence, as explained above, corresponds to $\mathcal{R}_{\mathrm{C}}$-compatibility. Archimedeanity of $\desirset[\succ]$, on the other hand, corresponds to the set of rules 
\begin{equation}\label{eq:abstractarchimedeanity}
\mathcal{R}_{\mathrm{A}}
\coloneqq
\big\{
  (\{\{\opt\}\},\{\opt-\epsilon\colon\epsilon\in\posreals\})
  \colon
  \opt\in\opts
\big\},
\end{equation}
where we identify $\epsilon\in\posreals$ with the constant gamble on $\states$ with value $\epsilon$. To impose both properties together, it suffices to consider the single set of rules $\mathcal{R}_{\mathrm{CA}}\coloneqq\mathcal{R}_{\mathrm{C}}\cup\mathcal{R}_{\mathrm{A}}$.

If besides coherence and Archimedeanity of $\desirset[\succ]$, we additionally impose that $\succ$ should be a weak order---or equivalently, that $\desirset[\succ]$ should be mixing---the representing lower previsions become linear expectations, resulting in orders that correspond to maximising expected utility with respect to a finitely additive probability measure~\citep{pmlr-v103-de-bock19b}. This can for example be achieved by imposing $\mathcal{R}_{\mathrm{CAM}}\coloneqq\mathcal{R}_{\mathrm{C}}\cup\mathcal{R}_{\mathrm{A}}\cup\mathcal{R}_{\mathrm{M}}$, where instead of $\mathcal{R}_{\mathrm{M}}$, we could have also used $\mathcal{R}_{\mathrm{W}}$ or $\mathcal{R}_{\mathrm{W2}}$.

As a final example, orders that are evenly continuous \citep{cozman2018:even:convexity} can be obtained by imposing $\mathcal{R}_{\mathrm{CE}}\coloneqq\mathcal{R}_{\mathrm{C}}\cup\mathcal{R}_{\mathrm{E}}$, with
\begin{equation*}
\mathcal{R}_{\mathrm{E}}
\coloneqq
\big\{
	(\{\{\opt[i]\}\colon i\in\naturals\}\cup\{\{\lim_{i\to\infty}(\lambda_i\altopt-\opt[i])\}\},\{\altopt\})
	\colon
\opt[i],\altopt\in\opts, \lambda_i\in\posreals
\big\}.
\end{equation*}

\section{Proper choice functions}

In order to motivate the use of choice functions of the form $\choicefun[\succ]$ and $\choicefun[\mathcal{O}]$, without directly assuming that they must be of this type, we now proceed to provide an axiomatic characterisation for such types of choice functions. We start by introducing the notion of a proper choice function. As we will see in Theorem~\ref{theo:properrepresentation}, these are exactly the choice functions of the form $\choicefun[\mathcal{O}]$.

To state the defining axioms of a proper choice function, we require some additional notation. First, for any $\optset\in\optsets$ and $w\in\opts$, we let $\optset+w\coloneqq\{\opt+w\colon\opt\in\optset\}$, and similarly for $\optset-w$. One of the properties of a proper choice function---see~\ref{ax:choice:translation} further on---will be that rejecting $\opt$ from $\optset$ is equivalent to rejecting $0$ from $\optset-\opt$. Similarly to how a proper order $\succ$ is completely determined by $\desirset[\succ]$, a proper choice function $\choicefun$ will therefore be completely determined by the sets from which it rejects zero. In particular, as we will see in Proposition~\ref{prop:CequivK}, the role of $\desirset[\succ]$ is now taken up by
\begin{equation}\label{eq:fromCtoK}
\rejectset[\choicefun]\coloneqq\{\optset\in\optsets\colon 0\notin\choicefun(\optset\cup\{0\}).
\end{equation}

Next, for any $\assessment\subseteq\optsets$, we denote by $\Phi_\assessment$ the collection of all maps $\phi\colon\assessment\to\opts$ that, for all $\optset\in\assessment$, select a single option $\phi\group{\optset}\in\optset$. 
Furthermore, for any such selection map $\phi\in\Phi_\assessment$, we 
let $\phi(\assessment)\coloneqq\{\phi(\optset)\colon\optset\in\assessment\}$ be the corresponding set of selected options. If $\assessment\neq\emptyset$, we also let $\Lambda_\assessment$ be the set of all maps $\lambda:\assessment\to\nonnegreals$ with \emph{finite} non-empty support, so $\lambda(\optset)>0$ for finitely many---but at least one---$\optset$ in $\assessment$ and $\lambda(\optset)=0$ otherwise. Properness for choice functions is now defined as follows.

\begin{definition}
\label{def:properchoicefunction}
A choice function $\choicefun$ is \emph{proper} if it satisfies the following properties:
\begin{enumerate}[label=$\mathrm{PC}_{\arabic*}$.,ref=$\mathrm{PC}_{\arabic*}$,leftmargin=*,start=0]
\item\label{ax:choice:singleton}
$\choicefun(\{\opt\})=\{\opt\}$ for all $\opt\in\opts$
\item\label{ax:choice:translation}
if $\opt\in\choicefun(\optset)$, then also $\opt+w\in\choicefun(\optset+w)$, for all $\optset\in\optsets$ and $\opt,w\in\opts$
\item\label{ax:choice:cone} if $\emptyset\neq\assessment\subseteq\rejectset[\choicefun]$ and if, for all $\phi\in\Phi_\assessment$, $\lambda_\phi\in\Lambda_\assessment$, then also
\begin{equation*}
\textstyle
\big\{\sum_{A\in\assessment}\lambda_\phi(A)\phi(A)\colon\phi\in\Phi_\assessment\big\}
\in\rejectset[\choicefun];
\end{equation*}
\item\label{ax:choice:mono} if $\opt\in\optset\subseteq\altoptset$ and $\opt\in\choicefun(\altoptset)$, then also $\opt\in\choicefun(\optset)$, for all $\optset,\altoptset\in\optsets$.
\end{enumerate}
\end{definition}

Each of these axioms can---but need not---be motivated by interpreting $\opt\in\choicefun(\optset)$ as `there is no $\altopt$ in $\optset$ that is better than $\opt$', where the `better than' relation satisfies the defining properties of a proper order. For \ref{ax:choice:singleton},~\ref{ax:choice:translation} and~\ref{ax:choice:mono}, this should be intuitively clear. For \ref{ax:choice:cone}, the starting point is that for any $\optset\in\rejectset[\choicefun]$, there is at least one $\altopt\in\optset$ that is better than zero. Hence, for any $\assessment\subseteq\rejectset$, there is at least one $\phi\in\Phi_{\assessment}$ such that every element of $\phi(\assessment)$ is better than zero. Combined with \ref{ax:order:transitive}--\ref{ax:order:addition}, or actually, \ref{ax:desirs:sum} and~\ref{ax:desirs:lambda}, it follows that $\sum_{A\in\assessment}\lambda_\phi(A)\phi(A)$ should be better than zero as well.

Every proper choice function $\choicefun$ is completely determined by a so-called proper set of option sets $\rejectset$, which is furthermore unique and equal to $\rejectset[\choicefun]$.

\begin{definition}
\label{def:coherence:infiniteoptionsets}
A \emph{set of option sets} is a---possibly empty---subset $\rejectset$ of $\optsets_{\emptyset}$. It is called \emph{proper} if it satisfies the following properties:
\begin{enumerate}[label=$\mathrm{PK}_{\arabic*}$.,ref=$\mathrm{PK}_{\arabic*}$,leftmargin=*,start=0]
\item\label{ax:infiniteoptionsets:noempty} $\emptyset\notin\rejectset$;
\item\label{ax:infiniteoptionsets:removezero} if $\optset\in\rejectset$ then also $\optset\setminus\set{0}\in\rejectset$;
\item\label{ax:infiniteoptionsets:cone} if $\emptyset\neq\assessment\subseteq\rejectset$ and if, for all $\phi\in\Phi_\assessment$, $\lambda_\phi\in\Lambda_\assessment$, then also
\begin{equation*}
\big\{\textstyle\sum_{A\in\assessment}\lambda_\phi(A)\phi(A)\colon\phi\in\Phi_\assessment\big\}
\in\rejectset;
\end{equation*}
\item\label{ax:infiniteoptionsets:mono} if $\optset\in\rejectset$ and $\optset\subseteq\altoptset\in\optsets$, then also $\altoptset\in\rejectset$.
\end{enumerate}
\end{definition}

\begin{proposition}\label{prop:CequivK}
A choice function $\choicefun$ is proper if and only if there is a proper set of option sets $\rejectset$ such that
\begin{equation}\label{eq:fromKtoC}
\choicefun(\optset)=
\big\{
	\opt\in\optset
	\colon
	\optset-\opt\notin\rejectset
\big\}
\text{ for all $\optset\in\optsets$}.
\end{equation}
This $\rejectset$ is then necessarily unique and equal to $\rejectset[\choicefun]$.
\end{proposition}




Hence, proper choice functions correspond to proper sets of option sets $\rejectset$. On the other hand, we know from Proposition~\ref{prop:orderequivD} that proper orders correspond to proper sets of options $\desirset$. Relating proper choice functions with proper orders, therefore, amounts to relating proper sets of option sets $\rejectset$ with (sets of) proper sets of options $\desirset$. The first step consists in associating with any set of options $\desirset$ a set of option sets
\begin{equation}\label{eq:infinite:desirset:to:rejectset}
\rejectset[\desirset]
\coloneqq\cset{\optset\in\optsets}{\optset\cap\desirset\neq\emptyset}.
\end{equation}
It is quite straightforward to show that $\desirset$ is proper if and only if $\rejectset[\desirset]$ is. 

\begin{proposition}\label{prop:DcoherentiffKDis}
Consider a set of options $\desirset$ and its corresponding set of option sets $\rejectset[\desirset]$. Then $\desirset$ is proper if and only if $\rejectset[\desirset]$ is.
\end{proposition}

What is perhaps more surprising though is that every proper set of option sets $\rejectset$ corresponds to a set $\setofdesirsets$ of proper sets of options $\desirset$.

\begin{theorem}\label{theo:infinite:coherentrepresentation:twosided}
A set of option sets $\rejectset$ is proper if and only if there is a non-empty set $\setofdesirsets\subseteq\desirsets$ of proper sets of options such that $\rejectset=\bigcap\cset{\rejectset[\desirset]}{\desirset\in\setofdesirsets}$.
The largest such set $\setofdesirsets$ is then $\desirsets(\rejectset)\coloneqq\cset{\desirset\in\desirsets}{\rejectset\subseteq\rejectset[\desirset]}$.
\end{theorem}

To additionally guarantee that the option sets in $\setofdesirsets$ are $\mathcal{R}$-compatible with a given set of rules $\mathcal{R}$, we introduce a notion of $\mathcal{R}$-compatibility for proper sets of option sets and choice functions.

\begin{definition}\label{def:abstractaxiom:K}
Let $\mathcal{R}$ be a set of rules. For any proper set of option sets $\rejectset$, we then say that $\rejectset$ is $\mathcal{R}$-compatible if for all $(\mathcal{A},B)\in\mathcal{R}$:
\begin{enumerate}[label=$\mathrm{PK}_{\mathrm{\mathcal{R}}}$.,ref=$\mathrm{PK}_{\mathrm{\mathcal{R}}}$,leftmargin=*]
\item\label{ax:infiniteoptionsets:abstractaxiom} if $A\in\rejectset$ for all $\optset\in\mathcal{A}$, then also $B\in\rejectset$.
\end{enumerate}
A proper choice function $\choicefun$ is called $\mathcal{R}$-compatible if $\rejectset[\choicefun]$ is $\mathcal{R}$-compatible.
\end{definition}

Clearly, a proper set of options $\desirset$ is $\mathcal{R}$-compatible if and only if $\rejectset[\desirset]$ is. What is far less obvious though, is that for sets of rules $\mathcal{R}$ that are \emph{monotone} a proper set of option sets $\rejectset$ is $\mathcal{R}$-compatible if and only if it corresponds to a set $\setofdesirsets$ of proper $\mathcal{R}$-compatible sets of options $\desirset$.

\begin{definition}\label{def:monotonesetofrules}
A set of rules $\mathcal{R}$ is monotone if for all $M\in\optsets$:
\begin{equation*}
(\mathcal{A},B)\in\mathcal{R}~\text{and}~\mathcal{A}\neq\emptyset
~\then~
(\{A\cup M\colon A\in\mathcal{A}\},B\cup M)\in\mathcal{R}.
\end{equation*}
\end{definition}

\begin{theorem}\label{theo:infinite:abstractrepresentation:twosided}
Let $\mathcal{R}$ be a monotone set of rules.
A set of option sets $\rejectset$ is then proper and $\mathcal{R}$-compatible if and only if there is a non-empty set $\setofdesirsets\subseteq\desirsets_{\mathcal{R}}
$ of proper $\mathcal{R}$-compatible sets of options such that $\rejectset=\bigcap\cset{\rejectset[\desirset]}{\desirset\in\setofdesirsets}$. 
The largest such set $\setofdesirsets$ is then $\desirsets_{\mathcal{R}}\group{\rejectset}\coloneqq\cset{\desirset\in\desirsets_{\mathcal{R}}}{\rejectset\subseteq\rejectset[\desirset]}$.
\end{theorem}
In practice, however, most sets of rules that one would like to impose on the representing sets of options are not monotone. From all the sets of rules that we considered in Section 3, for example, the only monotone ones are $\mathcal{R}_{\mathrm{C}}$ and $\mathcal{R}_{\mathrm{T}}$; trivially, in fact, since $\assessment=\emptyset$ for the rules in these sets.

Fortunately, a characterisation similar to Theorem~\ref{theo:infinite:abstractrepresentation:twosided} can also be obtained for sets of rules $\mathcal{R}$ that are not monotone. It suffices to replace $\mathcal{R}$ with its monotonification
\begin{equation*}
\mathrm{mon}(\mathcal{R})
\coloneqq
\mathcal{R}\cup
\big\{
(\{\optset\cup M\colon\optset\in\mathcal{A}\},\altoptset\cup M)\colon(\mathcal{A},\altoptset)\in\mathcal{R},\mathcal{A}\neq\emptyset,M\in\optsets
\big\},
\end{equation*}which is the smallest monotone set of rules that includes $\mathcal{R}$. This still leads to a representation in terms of proper $\mathcal{R}$-compatible sets of options because $\mathcal{R}$- and $\mathrm{mon}(\mathcal{R})$-compatibility are equivalent for sets of options. 


\begin{proposition}
\label{prop:monotonificationmakesnodifference}
Consider a set of rules $\mathcal{R}$ and a proper set of options $\desirset$. Then $\desirset$ is $\mathcal{R}$-compatible if and only if it is $\mathrm{mon}(\mathcal{R})$-compatible.
\end{proposition}

\begin{theorem}\label{theo:infinite:abstractrepresentation:twosided:monotonification}
Let $\mathcal{R}$ be a set of rules.
A set of option sets $\rejectset$ is then proper and $\mathrm{mon}(\mathcal{R})$-compatible if and only if there is a non-empty set $\setofdesirsets\subseteq\desirsets_{\mathcal{R}}
$ of proper $\mathcal{R}$-compatible sets of options such that $\rejectset=\bigcap\cset{\rejectset[\desirset]}{\desirset\in\setofdesirsets}$. 
The largest such set $\setofdesirsets$ is then $\desirsets_{\mathcal{R}}\group{\rejectset}\coloneqq\cset{\desirset\in\desirsets_{\mathcal{R}}}{\rejectset\subseteq\rejectset[\desirset]}$.
\end{theorem}


Putting the pieces together---combining Theorem~\ref{theo:infinite:abstractrepresentation:twosided:monotonification} with Propositions~\ref{prop:orderequivD} and~\ref{prop:CequivK}---we arrive at our main result: an axiomatic characterisation for choice functions of the form $\choicefun[\mathcal{O}]$, with $\mathcal{O}$ a set of $\mathcal{R}$-compatible proper orders.

\begin{theorem}\label{theo:properrepresentation}
Let $\mathcal{R}$ be a set of rules. 
Then a choice function $\choicefun$ is proper and $\mathrm{mon}(\mathcal{R})$-compatible if and only if there is a non-empty set $\mathcal{O}\subseteq\mathbf{O}_{\mathcal{R}}
$ of $\mathcal{R}$-compatible proper orders such that $\choicefun=\choicefun[\mathcal{O}]$. 
The largest such set is then
\begin{equation*}
\mathbf{O}_{\mathcal{R}}(\choicefun)\coloneqq\{\succ\,\in\mathbf{O}_{\mathcal{R}}\colon\choicefun[\succ](\optset)\subseteq\choicefun(\optset)\text{ for all $\optset\in\optsets$}\}.
\end{equation*}
\end{theorem}
A similar result holds without $\mathcal{R}$-compatibility as well. It corresponds to the special case $\mathcal{R}=\mathrm{mon}(\mathcal{R})=\emptyset$, for which $\mathcal{R}$- and $\mathrm{mon}(\mathcal{R})$-compatibility are trivially satisfied.

In order to obtain an axiomatic characterisation for choice functions of the form $\choicefun[\succ]$, with $\succ$ a single $\mathcal{R}$-compatible proper order, we additionally impose that $\choicefun$ is completely determined by its pairwise choices. Replacing $\mathcal{R}$ by $\mathrm{mon}(\mathcal{R})$ is not needed in this case.

\begin{definition}
A choice function $\choicefun$ is \emph{binary} if for all $\optset\in\optsets$ and $\opt\in\optset$:
\begin{equation*}
\opt\in\choicefun(\optset)
\Leftrightarrow
(\forall\altopt\in\optset\setminus\{\opt\})~\opt\in\choicefun(\{\opt,\altopt\})
\end{equation*}
\end{definition}

\begin{theorem}\label{theo:properrepresentationone}
Let $\mathcal{R}$ be a set of rules. Then a choice function $\choicefun$ is proper, binary and $\mathcal{R}$-compatible if and only if there is an $\mathcal{R}$-compatible proper order $\succ$ such that $\choicefun=\choicefun[\succ]$. This order is unique and characterised by
\begin{equation}\label{eq:properrepresentationone}
\altopt\succ\opt
\Leftrightarrow
\opt\notin\choicefun(\{\opt,\altopt\})
\text{ for all $\opt,\altopt\in\opts$.}
\end{equation}
and it furthermore satisfies
$\rejectset[\choicefun]=\rejectset[{\desirset[\succ]}]$.
\end{theorem}
Here too, a version without $\mathcal{R}$-compatibility is easily obtained by setting $\mathcal{R}=\mathrm{mon}(\mathcal{R})=\emptyset$.

\section{Conclusions and future work}

The main conclusion of this paper is that decision making based on (sets of) strict partial orders is completely characterised by specific properties of the resulting choice functions. That is, any choice function $\choicefun$ that satisfies these properties is guaranteed to correspond to a (set of) strict partial order(s). By additionally imposing a set of rules on $\choicefun$, we can furthermore guaruantee that the representing orders are of a particular type. 
As discussed in Section~\ref{sec:abstractaxiom}, this includes total orders, weak orders, orders based on coherent lower previsions, orders based on probability measures (such as maximising expected utility) and evenly continuous orders, depending on the set of rules that is imposed.

In future work, I intend to study if there are other types of orders that correspond to a set of rules, as well as explain in more detail why the types of orders and sets of rules in Section~\ref{sec:abstractaxiom} indeed correspond to one another. I also intend to further elaborate on the implications of Theorems~\ref{theo:properrepresentation} and~\ref{theo:properrepresentationone}, demonstrate how they can be applied in various contexts, and establish connections with earlier axiomatic characterisations for decision methods that are based on (sets of) orders~\citep{savage1972,nau1992,seidenfeld1995,nau2006,seidenfeld2010,pmlr-v103-de-bock19b,ipmu2020debock:arxiv}


%
\begin{acknowledgement}
This work was funded by the BOF starting grant 01N04819. I also thank Gert de Cooman, with whom I am developing a theory of choice functions for finite option sets~\citep{debock2018,pmlr-v103-de-bock19b}
. My continuous discussions with him have been very helpful in developing a theory of choice functions that allows choosing from infinite option sets, as I do here.
\end{acknowledgement}
%

	 
\bibliographystyle{spbasic}
\bibliography{infiniteoptionsets}
%
	
	

\iftoggle{arxiv}{

\appendix

\section*{Proofs of our main results}

\begin{proof}{\textit{\textbf{of Proposition~\ref{prop:orderequivD} }}}
For the `if'-part, we consider a binary relation $\succ$ on $\opts$ and a proper set of options $\desirset$ that satisfies Equation~\eqref{eq:fromDtoO}. That $\desirset$ is equal to $\desirset[\succ]$ (and is hence uniquely determined by $\succ$) follows from the fact that, for any $\opt\in\opts$,
\begin{equation*}
\opt\in\desirset[\succ]
\Leftrightarrow
\opt\succ0
\Leftrightarrow
\opt-0\in\desirset
\Leftrightarrow
\opt\in\desirset,
\end{equation*}
using Equation~\eqref{eq:fromDtoO} for the second equivalence. That $\succ$ satisfies~\ref{ax:order:addition} follows directly from Equation~\eqref{eq:fromDtoO}. That $\succ$ satisfies~\ref{ax:order:irreflexive}, \ref{ax:order:transitive} and~\ref{ax:order:multiplication} follows directly from Equation~\eqref{eq:fromDtoO} and the fact that $\desirset$ satisfies~\ref{ax:desirs:nozero},~\ref{ax:desirs:sum} and~\ref{ax:desirs:lambda}, respectively. 

For the `only if'-part, we consider a binary relation $\succ$ on $\opts$ that is a proper order and prove that there is a proper set of options $\desirset$ that satisfies Equation~\eqref{eq:fromDtoO}. In particular, we will prove that this is the case for $\desirset=\desirset[\succ]$. That $\desirset[\succ]$ satisfies Equation~\eqref{eq:fromDtoO} follows from the fact that, for any $\opt,\altopt\in\opts$:
\begin{equation*}
\opt\succ\altopt
\Leftrightarrow
\opt-\altopt\succ0
\Leftrightarrow
\opt-\altopt\in\desirset[\succ],
\end{equation*}
where the first equivalence follows from~\ref{ax:order:addition} (with $w=-\altopt$ and $w=\altopt$). That $\desirset[\succ]$ satisfies~\ref{ax:desirs:nozero} and~\ref{ax:desirs:lambda} follows directly from Equation~\eqref{eq:fromOtoD} and the fact that $\succ$ satisfies~\ref{ax:order:irreflexive} and~\ref{ax:order:multiplication}, respectively. It remains to show that $\desirset[\succ]$ satisfies~\ref{ax:desirs:sum}. So consider any $\opt$ and $\altopt$ in $\desirset[\succ]$, implying that $\opt\succ0$ and $\altopt\succ0$. Since $\opt\succ0$, \ref{ax:order:addition} implies that also $\opt+\altopt\succ\altopt$. Since $\opt+\altopt\succ\altopt$ and $\altopt\succ0$, \ref{ax:order:transitive} implies that $\opt+\altopt\succ0$ and therefore, that $\opt+\altopt\in\desirset[\succ]$, as desired.
\end{proof}

\begin{proof}{\textit{\textbf{of Proposition~\ref{prop:CequivK} }}}
For the `if'-part, we consider a choice function $\choicefun$ and a proper set of option sets $\rejectset$ that satisfies Equation~\eqref{eq:fromKtoC}. To see why $\rejectset$ is equal to $\rejectset[\choicefun]$ (and is hence uniquely determined by $\choicefun$), first observe that for any $\optset\in\optsets$,
\begin{align*}
\optset\in\rejectset[\choicefun]
\Leftrightarrow
0\notin\choicefun(\optset\cup\{0\})
\Leftrightarrow
(\optset\cup\{0\})-0\in\rejectset
\Leftrightarrow
\optset\cup\{0\}\in\rejectset
\Leftrightarrow
\optset\in\rejectset,
\end{align*}
where the first two equivalences follow from Equations~\eqref{eq:fromCtoK} and~\eqref{eq:fromKtoC} and the last equivalence follows from~\ref{ax:infiniteoptionsets:removezero} and~\ref{ax:infiniteoptionsets:mono}. Since $\emptyset\notin\rejectset$ because of~\ref{ax:infiniteoptionsets:noempty} and $\emptyset\notin\rejectset[\choicefun]$ because of Equation~\eqref{eq:fromCtoK}, this implies that $\optset\in\rejectset[\choicefun]\Leftrightarrow\optset\in\rejectset$ for all $\optset\in\optsets_{\emptyset}$. Hence, $\rejectset=\rejectset[\choicefun]$.
Next, we prove that $\choicefun$ is proper. Since $\rejectset=\rejectset[\choicefun]$ and $\rejectset$ satisfies~\ref{ax:infiniteoptionsets:cone} (since $\rejectset$ is proper), we immediately find that $\choicefun$ satisfies \ref{ax:choice:mono}.
Consider now any $\opt\in\opts$. We know from \ref{ax:infiniteoptionsets:noempty} that $\emptyset\notin\rejectset$ and therefore, because of \ref{ax:infiniteoptionsets:removezero}, that also $\{\opt\}-\opt=\{0\}\notin\rejectset$. Taking into account Equation~\eqref{eq:fromKtoC}, it follows that $\choicefun(\{\opt\})=\{\opt\}$. Hence, $\choicefun$ satisfies~\ref{ax:choice:singleton}. That $\choicefun$ also satisfies \ref{ax:choice:translation} follows directly from Equation~\eqref{eq:fromKtoC}. Indeed, for any $\optset\in\optsets$ and $\opt,w\in\opts$:
\begin{equation*}
\opt\in\choicefun(\optset)
\then
\optset-\opt\notin\rejectset
\then
(\optset+w)-(\opt+w)\notin\rejectset
\then
\opt+w\in\choicefun(\optset+w).
\end{equation*}
To see that $\choicefun$ satisfies \ref{ax:choice:mono}, it suffices to consider any $\opt\in\opts$ and $\optset,\altoptset\in\optsets$ such that $\opt\in\optset\subseteq\altoptset$ and $\opt\in\choicefun(\altoptset)$. It then follows from Equation~\eqref{eq:fromKtoC} that $\altoptset-\opt\notin\rejectset$. Since $\optset-\opt\subseteq\altoptset-\opt$, \ref{ax:infiniteoptionsets:mono} therefore implies that $\optset-\opt\notin\rejectset$. A final application of Equation~\eqref{eq:fromKtoC} therefore yields $\opt\in\choicefun(\optset)$, as desired.

For the `only if'-part, we consider a proper choice function $\choicefun$ and prove that there is a proper set of option sets $\rejectset$ that satisfies Equation~\eqref{eq:fromKtoC}. In particular, we will prove that this is the case for $\rejectset=\rejectset[\choicefun]$. 
That $\rejectset[\choicefun]$ satisfies Equation~\eqref{eq:fromKtoC} follows from the fact that, for any $\optset\in\optsets$ and $\opt\in\optset$:
\begin{equation*}
\opt\in\choicefun(\optset)
\Leftrightarrow
0\in\choicefun(\optset-\opt)
\Leftrightarrow
0\in\choicefun((\optset-\opt)\cup\{0\})
\Leftrightarrow
\optset-\opt\in\rejectset[\choicefun],
\end{equation*}
where the first equivalence follows from~\ref{ax:choice:translation} (with $w=-\opt$ and $w=\opt$) and the second follows from the fact that $0\in\optset-\opt$ (because $\opt\in\optset$). It therefore remains to prove that $\rejectset[\choicefun]$ is a proper set of option sets. That $\rejectset[\choicefun]$ satisfies \ref{ax:infiniteoptionsets:cone} follows directly from the fact that $\choicefun$ satisfies~\ref{ax:choice:cone}. That $\rejectset[\choicefun]$ satisfies~\ref{ax:infiniteoptionsets:noempty} follows from Equation~\eqref{eq:fromCtoK} and the fact that $\emptyset\notin\optsets$. To show that $\rejectset[\choicefun]$ satisfies~\ref{ax:infiniteoptionsets:removezero}, we consider any $\optset\in\rejectset[\choicefun]$. Since~\ref{ax:choice:singleton} implies that $0\in\choicefun(\{0\})=\choicefun(\{0\}\cup\{0\})$, we know from Equation~\eqref{eq:fromCtoK} that $\{0\}\notin\rejectset[\choicefun]$, and therefore, since $\optset\in\rejectset[\choicefun]$, that $\optset\setminus\{0\}\neq\emptyset$. Hence, $\optset\setminus\{0\}\in\optsets$. Furthermore, since $\optset\in\rejectset[\choicefun]$, it follows from Equation~\eqref{eq:fromCtoK} that $0\notin\choicefun(\optset\cup\{0\})$ and therefore, since $\optset\cup\{0\}=(\optset\setminus\{0\})\cup\{0\}$, that $0\notin\choicefun((\optset\setminus\{0\})\cup\{0\})$. Applying Equation~\eqref{eq:fromCtoK} once more therefore yields $\optset\setminus\{0\}\in\rejectset[\choicefun]$. Hence, $\rejectset[\choicefun]$ satisfies~\ref{ax:infiniteoptionsets:removezero}. To see that $\rejectset[\choicefun]$ satisfies \ref{ax:infiniteoptionsets:mono}, we consider any $\optset,\altoptset\in\optsets$ such that $\optset\subseteq\altoptset$ and $\optset\in\rejectset[\choicefun]$. It then follows from Equation~\eqref{eq:fromCtoK} that $0\notin\choicefun(\optset\cup\{0\})$. Since $\optset\cup\{0\}\subseteq\altoptset\cup\{0\}$, it therefore follows from~\ref{ax:choice:mono} that $0\notin\choicefun(\altoptset\cup\{0\})$, or equivalently, that $\altoptset\in\rejectset[\choicefun]$.
\end{proof}

\begin{proof}{\textit{\textbf{of Proposition~\ref{prop:DcoherentiffKDis} }}}
Consider any set of options $\desirset$ and let $\rejectset[\desirset]$ be its corresponding set of option sets.

For the `only if'-part of the statement, we assume that $\desirset$ is proper and prove that $\rejectset[\desirset]$ is then proper as well. That $\rejectset[\desirset]$ satisfies~\ref{ax:infiniteoptionsets:noempty} follows from Equation~\eqref{eq:infinite:desirset:to:rejectset} and the fact that $\emptyset\notin\optsets$. That $\rejectset[\desirset]$ satisfies \ref{ax:infiniteoptionsets:removezero} follows from Equation~\eqref{eq:infinite:desirset:to:rejectset} and~\ref{ax:desirs:nozero}. That $\rejectset[\desirset]$ satisfies~\ref{ax:infiniteoptionsets:mono} is immediate from Equation~\eqref{eq:infinite:desirset:to:rejectset}. To see that $\rejectset[\desirset]$ also satisfies~\ref{ax:infiniteoptionsets:cone}, we consider any $\emptyset\neq\assessment\subseteq\rejectset[\desirset]$ and, for all $\phi\in\Phi_\assessment$, some $\lambda_\phi\in\Lambda_\assessment$. For any $\optset\in\assessment$, since $\optset\in\rejectset[\desirset]$, there is some $\opt^*\in\optset\cap\desirset$, which we denote by $\phi^*(\optset)$. For the resulting map $\phi^*\in\Phi_{\assessment}$, we have that $\phi^*(\optset)\in\desirset$ for all $\optset\in\assessment$, and therefore, because of~\ref{ax:desirs:sum} and~\ref{ax:desirs:lambda}, that $\sum_{\optset\in\assessment}\lambda_{\phi^*}(\optset)\phi^*(\optset)\in\desirset$. Hence,
\begin{equation*}
\big\{\textstyle\sum_{A\in\assessment}\lambda_\phi(A)\phi(A)\colon\phi\in\Phi_\assessment\big\}
\in\rejectset[\desirset].
\end{equation*}

For the `if'-part of the statement, we assume that $\rejectset[\desirset]$ is proper and prove that $\desirset$ is then proper as well. To prove that $\desirset$ satisfies~\ref{ax:desirs:nozero}, we assume \emph{ex absurdo} that $0\in\desirset$. This implies that $\{0\}\in\rejectset[\desirset]$ and therefore, because of~\ref{ax:infiniteoptionsets:removezero}, that $\emptyset\in\rejectset[\desirset]$, contradicting~\ref{ax:infiniteoptionsets:noempty}. To prove that $\desirset$ satisfies~\ref{ax:desirs:sum}, we consider any $\opt,\altopt\in\desirset$. It then follows from Equation~\eqref{eq:infinite:desirset:to:rejectset} that $\{\opt\}\in\rejectset[\desirset]$ and $\{\altopt\}\in\rejectset[\desirset]$. Now let $\assessment=\{\{\opt\},\{\altopt\}\}$. Then $\emptyset\neq\assessment\subseteq\rejectset[\desirset]$ and $\Phi_{\assessment}$ contains only a single function $\phi^*$, defined by $\phi^*(\{\opt\})=\opt$ and $\phi^*(\{\altopt\})=\altopt$. Let $\lambda_{\phi^*}\in\Lambda_{\assessment}$ be defined by $\lambda_{\phi^*}(\{\opt\})=\lambda_{\phi^*}(\{\altopt\})\coloneqq1$. It then follows from~\ref{ax:infiniteoptionsets:cone} that 
\begin{align*}
\{\opt+\altopt\}
&=
\big\{\lambda_{\phi^*}(\{\opt\})\phi^*(\{\opt\})
+ \lambda_{\phi^*}(\{\altopt\})\phi^*(\{\altopt\})\big\}\\
&=
\big\{\textstyle\sum_{\optset\in\assessment}
\lambda_{\phi^*}(\optset)\phi^*(\optset)
\big\}
=
\big\{\textstyle\sum_{\optset\in\assessment}
\lambda_{\phi}(\optset)\phi(\optset)
\colon
\phi\in\Phi_{\assessment}
\big\}
\in\rejectset[\desirset],
\end{align*}
or equivalently, that $\opt+\altopt\in\desirset$. The proof for \ref{ax:desirs:lambda} is similar. 
Consider any $\opt\in\desirset$ and $\lambda\in\posreals$. It then follows from Equation~\eqref{eq:infinite:desirset:to:rejectset} that $\{\opt\}\in\rejectset[\desirset]$. Now let $\assessment=\{\{\opt\}\}$. Then $\emptyset\neq\assessment\subseteq\rejectset[\desirset]$ and $\Phi_{\assessment}$ contains only a single function $\phi^*$, defined by $\phi^*(\{\opt\})=\opt$. Let $\lambda_{\phi^*}\in\Lambda_{\assessment}$ be defined by $\lambda_{\phi^*}(\{\opt\})\coloneqq\lambda$. It then follows from~\ref{ax:infiniteoptionsets:cone} that 
\begin{align*}
\{\lambda\opt\}
&=
\big\{\lambda_{\phi^*}(\{\opt\})\phi^*(\{\opt\})\big\}\\
&=
\big\{\textstyle\sum_{\optset\in\assessment}
\lambda_{\phi^*}(\optset)\phi^*(\optset)
\big\}
=
\big\{\textstyle\sum_{\optset\in\assessment}
\lambda_{\phi}(\optset)\phi(\optset)
\colon
\phi\in\Phi_{\assessment}
\big\}
\in\rejectset[\desirset],
\end{align*}
or equivalently, that $\lambda\opt\in\desirset$.
\end{proof}

\begin{lemma}\label{lem:intersectionofproperKsisproper}
Consider any non-empty set $\mathcal{K}$ of proper sets of option sets. Then $\rejectset[\mathcal{K}]\coloneqq\bigcap\{\rejectset\colon\rejectset\in\mathcal{K}\}$ is a proper set of options sets as well.
\end{lemma}
\begin{proof}
Observe that for any $\optset\in\optsets_{\emptyset}$, since $\mathcal{K}$ is non-empty, we have that $\optset\in\rejectset[\mathcal{K}]$ if and only $\optset\in\rejectset$ for all $\rejectset\in\mathcal{K}$. Given this observation, the statement follows directly from Definition~\ref{def:coherence:infiniteoptionsets}. For example, for every $\rejectset\in\mathcal{K}$, we know from~\ref{ax:infiniteoptionsets:noempty} that $\emptyset\notin\rejectset$. Hence, $\emptyset\notin\rejectset[\mathcal{K}]$, so $\rejectset[\mathcal{K}]$ satisfies~\ref{ax:infiniteoptionsets:noempty}. As a second example, if $\optset\in\rejectset[\mathcal{K}]$, then $\optset\in\rejectset$ for all $\rejectset\in\mathcal{K}$. Since every $\rejectset\in\mathcal{K}$ is proper, it follows from~\ref{ax:infiniteoptionsets:removezero} that $\optset\setminus\{0\}\in\rejectset$ for all $\rejectset\in\mathcal{K}$. Hence, $\optset\setminus\{0\}\in\rejectset[\mathcal{K}]$. We conclude that $\rejectset[\mathcal{K}]$ satisfies~\ref{ax:infiniteoptionsets:removezero}. The proof for~\ref{ax:infiniteoptionsets:cone} and~\ref{ax:infiniteoptionsets:mono} is completely analogous.
\end{proof}

\begin{proof}{\textit{\textbf{of Theorem~\ref{theo:infinite:coherentrepresentation:twosided} }}}
For the `if'-part, consider any non-empty set $\setofdesirsets\subseteq\desirsets$ of proper sets of options such that $\rejectset=\bigcap\cset{\rejectset[\desirset]}{\desirset\in\setofdesirsets}$. Then for any $\desirset\in\setofdesirsets$, $\rejectset[\desirset]$ is a proper set of option sets because of Proposition~\ref{prop:DcoherentiffKDis}. 
Since $\setofdesirsets$ is non-empty, it therefore follows from Lemma~\ref{lem:intersectionofproperKsisproper} that $\rejectset=\bigcap\cset{\rejectset[\desirset]}{\desirset\in\setofdesirsets}$ is proper set of option sets as well.

For the `only if' part, we assume that $\rejectset$ is a proper set of option sets.
We need to prove that there is a non-empty set $\setofdesirsets\subseteq\desirsets$ of proper sets of options such that $\rejectset=\bigcap\cset{\rejectset[\desirset]}{\desirset\in\setofdesirsets}$, and that $\desirsets\group{\rejectset}$ is the largest such set. Clearly, if there is a set $\setofdesirsets\subseteq\desirsets$ such that $\rejectset=\bigcap\cset{\rejectset[\desirset]}{\desirset\in\setofdesirsets}$, then $\setofdesirsets$ must be a subset of $\desirsets\group{\rejectset}$. 
It therefore suffices to prove the statement for the particular set $\setofdesirsets\coloneqq\desirsets(\rejectset)$. That is, we will prove that $\desirsets(\rejectset)$ is non-empty and that $\rejectset=\bigcap\cset{\rejectset[\desirset]}{\desirset\in\desirsets(\rejectset)}$. To that end, we will prove that for any $\optset\in\optsets_{\emptyset}$ such that $\optset\notin\rejectset$, there is a set of options $\desirset\in\desirsets(\rejectset)$ such that $\optset\notin\rejectset[\desirset]$. On the one hand, this implies that $\bigcap\cset{\rejectset[\desirset]}{\desirset\in\desirsets(\rejectset)}\subseteq\rejectset$ and therefore, since the converse set inclusion holds trivially, that $\rejectset=\bigcap\cset{\rejectset[\desirset]}{\desirset\in\desirsets(\rejectset)}$. On the other hand, since it follows from \ref{ax:infiniteoptionsets:noempty} that there is at least one option set $\optset\in\optsets_{\emptyset}$ such that $\optset\notin\rejectset$---in particular, $\optset=\emptyset$---it also implies that $\desirsets(\rejectset)$ is non-empty.

So consider any $\optset\in\optsets_{\emptyset}$ such that $\optset\notin\rejectset$. We need to prove that there is some $\desirset\in\desirsets(\rejectset)$ such that $\optset\notin\rejectset[\desirset]$. We first consider the case $\rejectset=\emptyset$. In that case, let $\desirset=\emptyset$. It then follows from Equation~\eqref{eq:infinite:desirset:to:rejectset} that $\rejectset[\desirset]=\emptyset$, which implies that $\rejectset\subseteq\rejectset[\desirset]$ and $\optset\notin\rejectset[\desirset]$. Hence, since $\desirset=\emptyset$ is trivially a proper set of options, we conclude that $\desirset\in\desirsets(\rejectset)$ and $\optset\notin\rejectset[\desirset]$. 
It remains to consider the case $\rejectset\neq\emptyset$. In that case, assume \emph{ex absurdo} that $\optset\in\rejectset[\desirset]$ for every $\desirset\in\desirsets(\rejectset)$. We will show that this leads to a contradiction.

Consider any $\phi\in\Phi_{\rejectset}$ and let
\begin{equation}\label{eq:posiofphiK}
\desirset[\phi]\coloneqq\big\{\textstyle\sum_{\altoptset\in\rejectset}\lambda(\altoptset)\phi(\altoptset)\colon \lambda\in\Lambda_{\rejectset}\big\}
\end{equation}
be the set of all finite positive linear combinations of options in $\phi(\rejectset)$. 
Then for any $\altoptset\in\rejectset$, we have that $\phi(\altoptset)\in\desirset[\phi]$---just let $\lambda(\altoptset)\coloneqq1$ and $\lambda(\altoptset')\coloneqq0$ for all $\altoptset'\in\rejectset\setminus\{\altoptset\}$---and therefore, since $\phi(\altoptset)\in\altoptset$, also that $\altoptset\in\rejectset[{\desirset[\phi]}]$. Hence, $\rejectset\subseteq\rejectset[{\desirset[\phi]}]$.
Furthermore, by construction, $\desirset[\phi]$ clearly satisfies~\ref{ax:desirs:lambda} and~\ref{ax:desirs:sum}.

We now consider two cases: $0\notin\desirset[\phi]$ and $0\in\desirset[\phi]$. If $0\notin\desirset[\phi]$, then $\desirset[\phi]$ satisfies~\ref{ax:desirs:nozero}. Since it also satisfies
\ref{ax:desirs:lambda} and \ref{ax:desirs:sum}, we therefore find that $\desirset[\phi]$ is a proper set of options. Since $\rejectset\subseteq\rejectset[{\desirset[\phi]}]$, this implies that $\desirset[\phi]\in\desirsets(\rejectset)$. By our assumption, it then follows that $\optset\in\rejectset[{\desirset[\phi]}]$, implying that $\optset\cap\desirset[\phi]\neq\emptyset$ and therefore also $(\optset\cup\{0\})\cap\desirset[\phi]\neq\emptyset$. If $0\in\desirset[\phi]$, then $0\in(\optset\cup\{0\})\cap\desirset[\phi]$, implying once more that $(\optset\cup\{0\})\cap\desirset[\phi]\neq\emptyset$. Hence, in both cases, we find that $(\optset\cup\{0\})\cap\desirset[\phi]\neq\emptyset$. Now let $\opt[\phi]$ be any element of $(\optset\cup\{0\})\cap\desirset[\phi]$. Since $\opt[\phi]\in\desirset[\phi]$, it follows from Equation~\eqref{eq:posiofphiK} that there is some $\lambda_\phi\in\Lambda_{\rejectset}$ such that $\sum_{\altoptset\in\rejectset}\lambda_\phi(\altoptset)\phi(\altoptset)=\opt[\phi]\in\optset\cup\{0\}$.

In summary then, for any $\phi\in\Phi_{\rejectset}$, we have found some $\lambda_\phi\in\Lambda_{\rejectset}$ such that $\sum_{\altoptset\in\rejectset}\lambda_\phi(\altoptset)\phi(\altoptset)\in\optset\cup\{0\}$. Now let $A^*\coloneqq\{\sum_{\altoptset\in\rejectset}\lambda_\phi(\altoptset)\phi(\altoptset)\colon\phi\in\Phi_{\rejectset}\}$. Since $\rejectset\neq\emptyset$, it then follows from~\ref{ax:infiniteoptionsets:cone} that $A^*\in\rejectset$. Since $A^*$ is by construction a subset of $\optset\cup\{0\}$,~\ref{ax:infiniteoptionsets:mono} therefore implies that $\optset\cup\{0\}\in\rejectset$. This is however impossible because $\optset\notin\rejectset$. If $0\in\optset$, this contradiction is trivial. If $0\notin\optset$, this contradiction follows from~\ref{ax:infiniteoptionsets:removezero}.
\end{proof}

\begin{lemma}\label{lem:intersectionofproperDsisproper}
Consider any non-empty set $\mathcal{D}$ of proper sets of options. Then $\desirset[\mathcal{D}]\coloneqq\bigcap\{\desirset\colon\desirset\in\mathcal{D}\}$ is a proper set of options as well.
\end{lemma}
\begin{proof}
Observe that for any $\opt\in\opts$, since $\mathcal{D}$ is non-empty, we have that $\opt\in\desirset[\mathcal{D}]$ if and only if $\opt\in\desirset$ for all $\desirset\in\mathcal{D}$. Given this observation, the statement follows directly from Definition~\ref{def:propdesir}.
\end{proof}

\begin{proof}{\textit{\textbf{of Theorem~\ref{theo:infinite:abstractrepresentation:twosided} }}}
For the `if'-part of the statement, we assume that there is a non-empty set $\setofdesirsets\subseteq\desirsets_{\mathcal{R}}$ of proper $\mathcal{R}$-compatible sets of options such that $\rejectset=\bigcap\cset{\rejectset[\desirset]}{\desirset\in\setofdesirsets}$. It then follows from Theorem~\ref{theo:infinite:coherentrepresentation:twosided} that $\rejectset$ is a proper set of option sets. To show that it is also $\mathcal{R}$-compatible, we consider any $(\mathcal{A},B)\in\mathcal{R}$ such that $A\in\rejectset$ for all $A\in\mathcal{A}$. Fix any $\desirset\in\setofdesirsets$. Then for all $A\in\mathcal{A}$, since $A\in\rejectset\subseteq\rejectset[\desirset]$, we know that $A\cap\desirset\neq\emptyset$. Since $\desirset$ is $\mathcal{R}$-compatible, this implies that $B\cap\desirset\neq\emptyset$, which in turn implies that $B\in\rejectset[\desirset]$. Since this is true for every $\desirset\in\setofdesirsets$, we find that $B\in\bigcap\cset{\rejectset[\desirset]}{\desirset\in\setofdesirsets}=\rejectset$. Hence, $\rejectset$ is $\mathcal{R}$-compatible.

For the `only if'-part, we assume that $\rejectset$ is a proper $\mathcal{R}$-compatible set of option sets. We need to prove that there is a non-empty set $\setofdesirsets\subseteq\desirsets_{\mathcal{R}}$ of proper $\mathcal{R}$-compatible sets of options such that $\rejectset=\bigcap\cset{\rejectset[\desirset]}{\desirset\in\setofdesirsets}$, and that $\desirsets_{\mathcal{R}}\group{\rejectset}$ is the largest such set. Clearly, if there is a set $\setofdesirsets\subseteq\desirsets_{\mathcal{R}}$ such that $\rejectset=\bigcap\cset{\rejectset[\desirset]}{\desirset\in\setofdesirsets}$, then $\setofdesirsets$ must be a subset of $\desirsets_{\mathcal{R}}\group{\rejectset}$. 
It therefore suffices to prove the statement for the particular set $\setofdesirsets\coloneqq\desirsets_{\mathcal{R}}(\rejectset)$. That is, we will prove that $\desirsets_{\mathcal{R}}(\rejectset)$ is non-empty and that $\rejectset=\bigcap\cset{\rejectset[\desirset]}{\desirset\in\desirsets_{\mathcal{R}}(\rejectset)}$. 
To that end, we will prove that for any $\optset\in\optsets_{\emptyset}$ such that $\optset\notin\rejectset$, there is a set of options $\desirset[\optset]\in\desirsets_{\mathcal{R}}(\rejectset)$ such that $\smash{\optset\notin\rejectset[{\desirset[\optset]}]}$. On the one hand, this implies that $\bigcap\cset{\rejectset[\desirset]}{\desirset\in\desirsets_{\mathcal{R}}\group{\rejectset}}\subseteq\rejectset$ and therefore, since the converse set inclusion holds trivially, that $\rejectset=\bigcap\cset{\rejectset[\desirset]}{\desirset\in\desirsets_{\mathcal{R}}\group{\rejectset}}$. On the other hand, since it follows from \ref{ax:infiniteoptionsets:noempty} that there is at least one option set $\optset\in\optsets_{\emptyset}$ such that $\optset\notin\rejectset$---in particular, $\optset=\emptyset$---it also implies that $\desirsets_{\mathcal{R}}(\rejectset)$ is non-empty. 
So consider any $\optset\in\optsets_{\emptyset}$ such that $\optset\notin\rejectset$.


Since $\rejectset$ is proper, Theorem~\ref{theo:infinite:coherentrepresentation:twosided} implies that there is a non-empty set $\setofdesirsets$ of proper sets of options such that $\smash{\rejectset=\bigcap\cset{\rejectset[\desirset]}{\desirset\in\setofdesirsets}}$. Since $\optset\notin\rejectset$, this implies that there must be some $\tilde{\desirset}\in\setofdesirsets$ such that $\rejectset\subseteq\rejectset[\tilde{\desirset}]$ and $\smash{\optset\notin\rejectset[\tilde{\desirset}]}$. 
Let $\desirsets$ be the set of all proper sets of options and consider the set $\smash{\setofdesirsets_{\tilde{\desirset}}\coloneqq\{\desirset\in\desirsets\colon\rejectset\subseteq\rejectset[\desirset]\subseteq\rejectset[\tilde{\desirset}]\}}$, partially ordered by set inclusion. In particular, for any two $\desirset[1],\desirset[2]\in\setofdesirsets_{\tilde{\desirset}}$, we say that $\desirset[2]$ dominates $\desirset[1]$, denoted by $\desirset[1]\sqsubset\desirset[2]$, if $\desirset[2]\subset\desirset[1]$; so subsets dominate their supersets. We will use Zorn's lemma to prove that this partially ordered set has a maximal---undominated---element. To do so, we must show that $\setofdesirsets_{\tilde{\desirset}}$ is non-empty and that any chain in $\setofdesirsets_{\tilde{\desirset}}$---any completely ordered subset of $\setofdesirsets_{\tilde{\desirset}}$---has an upper bound in $\setofdesirsets_{\tilde{\desirset}}$. That $\setofdesirsets_{\tilde{\desirset}}$ is non-empty follows from the fact that it contains $\smash{\tilde{\desirset}}$. So consider any chain $\{\desirset[i]\}_{i\in I}$ in $\setofdesirsets_{\tilde{\desirset}}$. Since $\setofdesirsets_{\tilde{\desirset}}$ is a non-empty set of proper sets of options, it follows from Lemma~\ref{lem:intersectionofproperDsisproper} that the intersection $\desirset[I]\coloneqq\cap_{i\in I}\desirset[i]$ is a proper set of options as well.
Furthermore, since $\rejectset\subseteq\rejectset[{\desirset[i]}]$ for all $i\in I$, we also have that $\rejectset\subseteq\rejectset[{\desirset[I]}]$. 
Hence, $\desirset[I]\in\setofdesirsets_{\tilde{\desirset}}$. Since $\desirset[I]$ is also by definition an upper bound for the chain $\{\desirset[i]\}_{i\in I}$, in the sense that $\desirset[i]\sqsubseteq\desirset[I]$ for all $i\in I$, we conclude that $\{\desirset[i]\}_{i\in I}$ has an upper bound in $\setofdesirsets_{\tilde{\desirset}}$. Since this is true for every chain in $\setofdesirsets_{\tilde{\desirset}}$, and since $\setofdesirsets_{\tilde{\desirset}}$ is non-empty, it follows from Zorn's lemma that $\setofdesirsets_{\tilde{\desirset}}$ has a maximal element. 
Let $\desirset[\optset]$ be any such a maximal element. 

Since $\smash{\desirset[\optset]\in\setofdesirsets_{\tilde{\desirset}}}$, we know that $\desirset[\optset]$ is proper and that $\smash{\rejectset\subseteq\rejectset[{\desirset[\optset]}]\subseteq\rejectset[\tilde{\desirset}]}$. Since $\smash{\optset\notin\rejectset[\tilde{\desirset}]}$, this also implies that $\optset\notin\rejectset[{\desirset[\optset]}]$. The only thing left to prove, therefore, is that $\desirset[\optset]$ is $\mathcal{R}$-compatible. Assume \emph{ex absurdo} that it is not. We will show that this leads to a contradiction.

Since $\desirset[\optset]$ is proper but not $\mathcal{R}$-compatible, it follows from Definition~\ref{def:abstractaxiom:D} that there is some $(\mathcal{A}^*,B^*)\in\mathcal{R}$ such that $A^*\cap\desirset[\optset]\neq\emptyset$ for all $A^*\in\mathcal{A}$ but $B^*\cap\desirset[\optset]=\emptyset$.
There are now two possibilities: $\mathcal{A}^*=\emptyset$ and $\mathcal{A}^*\neq\emptyset$. As we will see, they both lead to a contradiction. If $\mathcal{A}^*=\emptyset$, it follows from the $\mathcal{R}$-compatibility of $\rejectset$ that $B^*\in\rejectset\subseteq\rejectset[{\desirset[\optset]}]$, contradicting the fact that $B^*\cap\desirset[\optset]=\emptyset$.
It remains to consider the case $\mathcal{A}^*\neq\emptyset$.
In that case, let $M\coloneqq\opts\setminus\desirset[\optset]$. Then $M\in\optsets$ because~\ref{ax:desirs:nozero} implies that $0\notin\desirset[\optset]$, so $0\in M$ and hence $M\neq\emptyset$. Since $(\mathcal{A}^*,B^*)\in\mathcal{R}$ and $\mathcal{A}^*\neq\emptyset$, the monotonicity of $\mathcal{R}$ implies that $(\{A^*\cup M\colon A^*\in\mathcal{A}^*\},B^*\cup M)\in\mathcal{R}$.
Consider now any $\desirset\in\setofdesirsets$ and any $A^*\in\mathcal{A}^*$. 
If $\desirset\subset\desirset[\optset]$, then $\rejectset\subseteq\rejectset[\desirset]\subset\rejectset[{\desirset[\optset]}]\subseteq\rejectset[\tilde{\desirset}]$, so $\desirset\in\setofdesirsets_{\tilde{\desirset}}$. This is impossible because $\desirset[\optset]$ is a maximal element of $\setofdesirsets_{\tilde{\desirset}}$ with respect to $\sqsubset$ and therefore a minimal element with respect to $\subset$. Hence, we find that either $\desirset=\desirset[\optset]$ or $\desirset\cap(\opts\setminus\desirset[\optset])\neq\emptyset$. If $\desirset=\desirset[\optset]$, then since $A^*\cap\desirset[\optset]\neq\emptyset$ and therefore $(A^*\cup M)\cap\desirset[\optset]\neq\emptyset$, we find that $A^*\cup M\in\rejectset[{\desirset[\optset]}]=\rejectset[\desirset]$. If $\desirset\cap(\opts\setminus\desirset[\optset])\neq\emptyset$, then $\desirset\cap M\neq\emptyset$ and therefore also $\desirset\cap(A^*\cup M)\neq\emptyset$, which implies that $A^*\cup M\in\rejectset[\desirset]$. Hence, in both cases, we find that $A^*\cup M\in\rejectset[\desirset]$. Since this is true for every $\desirset\in\setofdesirsets$ and $A^*\in\mathcal{A}^*$, and since $\smash{\rejectset=\bigcap\cset{\rejectset[\desirset]}{\desirset\in\setofdesirsets}}$, it follows that $A^*\cup M\in\rejectset$ for all $A^*\in\mathcal{A}^*$. 
Since $(\{A^*\cup M\colon A^*\in\mathcal{A}^*\},B^*\cup M)\in\mathcal{R}$, the $\mathcal{R}$-compatibility of $\rejectset$ therefore implies that $B^*\cup M\in\rejectset\subseteq\rejectset[{\desirset[\optset]}]$, so $(B^*\cup M)\cap\desirset[\optset]\neq\emptyset$. Since $B^*\cap\desirset[\optset]=\emptyset$, this implies that $M\cap\desirset[\optset]\neq\emptyset$, which is clearly impossible because $M=\opts\setminus\desirset[\optset]$. Hence, we find that $\desirset[\optset]$ must indeed be $\mathcal{R}$-compatible.
\end{proof}

\begin{proof}{\textit{\textbf{of Proposition~\ref{prop:monotonificationmakesnodifference} }}}
Since $\mathcal{R}$ is a subset of $\mathrm{mon}(\mathcal{R})$, we trivially have that $\mathrm{mon}(\mathcal{R})$-compatibility implies $\mathcal{R}$-compatibility. To prove the converse, asume that $\desirset$ is $\mathcal{R}$-compatible and consider any $(\mathcal{A},B)\in\mathrm{mon}(\mathcal{R})$ such that $A\cap\desirset\neq\emptyset$ for all $A\in\mathcal{A}$. We need to prove that $B\cap\desirset\neq\emptyset$. If $(\mathcal{A},B)\in\mathcal{R}$, this follows directly from the $\mathcal{R}$-compatibility of $\desirset$. Otherwise, since $(\mathcal{A},B)\in\mathrm{mon}(\mathcal{R})$, there are $(\mathcal{A}^*,B^*)\in\mathcal{R}$ and $M\in\optsets$ such that $\mathcal{A}=\{A^*\cup M\colon A^*\in\mathcal{A}^*\}$, $B=B^*\cup M$ and $\mathcal{A}^*\neq\emptyset$. Since $A\cap\desirset\neq\emptyset$ for all $A\in\mathcal{A}$, this implies that $(A^*\cup M)\cap\desirset\neq\emptyset$ for all $A^*\in\mathcal{A}^*$. We now consider two cases: $M\cap\desirset\neq\emptyset$ and $M\cap\desirset=\emptyset$. If $M\cap\desirset\neq\emptyset$, then also $B\cap\desirset=(B^*\cup M)\cap\desirset\neq\emptyset$, as desired. If $M\cap\desirset=\emptyset$, then for all $A^*\in\mathcal{A}^*$, we infer from $(A^*\cup M)\cap\desirset\neq\emptyset$ that $A^*\cap\desirset\neq\emptyset$. Since $\desirset$ is $\mathcal{R}$-compatible and $(\mathcal{A}^*,B^*)\in\mathcal{R}$, this implies that $B^*\cap\desirset\neq\emptyset$. Hence, also in this second case, $B\cap\desirset=(B^*\cup M)\cap\desirset\neq\emptyset$. 
\end{proof}

\begin{proof}{\textit{\textbf{of Theorem~\ref{theo:infinite:abstractrepresentation:twosided:monotonification} }}}
Since $\mathrm{mon}(\mathcal{R})$ is clearly a monotone set of rules, we know from Theorem~\ref{theo:infinite:abstractrepresentation:twosided} that $\rejectset$ is proper and $\mathrm{mon}(\mathcal{R})$-compatible if and only if there is a non-empty set $\setofdesirsets\subseteq\desirsets_{\mathrm{mon}(\mathcal{R})}
$ such that $\rejectset=\bigcap\cset{\rejectset[\desirset]}{\desirset\in\setofdesirsets}$, and that the largest such set $\setofdesirsets$ is $\desirsets_{\mathrm{mon}(\mathcal{R})}\group{\rejectset}\coloneqq\cset{\desirset\in\desirsets_{\mathrm{mon}(\mathcal{R})}}{\rejectset\subseteq\rejectset[\desirset]}$. Since we know from Proposition~\ref{prop:monotonificationmakesnodifference} that $\desirsets_{\mathrm{mon}(\mathcal{R})}=\desirsets_{\mathcal{R}}$ and hence also $\desirsets_{\mathrm{mon}(\mathcal{R})}(\rejectset)=\desirsets_{\mathcal{R}}(\rejectset)$, this concludes the proof.
\end{proof}

\begin{proof}{\textit{\textbf{of Theorem~\ref{theo:properrepresentation} }}}
For the `if'-part of the statement, we assume that there is a non-empty set $\mathcal{O}\subseteq\mathbf{O}_{\mathcal{R}}$ of $\mathcal{R}$-compatible proper orders such that $\choicefun=\choicefun[\mathcal{O}]$. Let $\setofdesirsets\coloneqq\{\desirset[\succ]\colon\hspace{-2pt}\succ\,\in\hspace{-1pt}\raisebox{-0.7pt}{$\mathcal{O}$}\}$. 
Then $\setofdesirsets$ is non-empty because $\mathcal{O}$ is, and the sets of options in $\setofdesirsets$ are proper and $\mathcal{R}$-compatible because of Proposition~\ref{prop:orderequivD} and Definition~\ref{def:abstractaxiom:D}. Hence, $\setofdesirsets$ is a non-empty set of proper $\mathcal{R}$-compatible sets of options. It therefore follows from Theorem~\ref{theo:infinite:abstractrepresentation:twosided:monotonification} that $\rejectset\coloneqq\bigcap\{\rejectset[\desirset]\colon\desirset\in\setofdesirsets\}$ is a proper $\mathrm{mon}(\mathcal{R})$-compatible set of option sets. Now observe that for any $\optset\in\optsets$:
\begin{align*}
\choicefun[\mathcal{O}](\optset)
&=\{\opt\in\optset\colon(\exists \succ\,\in\raisebox{-0.7pt}{$\mathcal{O}$})\,(\nexists\altopt\in\optset)~\altopt\succ\opt\}\\
&=\{\opt\in\optset\colon(\exists \succ\,\in\raisebox{-0.7pt}{$\mathcal{O}$})\,(\nexists\altopt\in\optset)~\altopt-\opt\in\desirset[\succ]\}\\
&=\{\opt\in\optset\colon(\exists\desirset\in\setofdesirsets)\,(\nexists\altopt\in\optset)~\altopt-\opt\in\desirset\}\\
&=\{\opt\in\optset\colon(\exists\desirset\in\setofdesirsets)\,(\optset-\opt)\cap\desirset=\emptyset\}\\
&=\{\opt\in\optset\colon(\exists\desirset\in\setofdesirsets)\,\optset-\opt\notin\rejectset[\desirset]\}\\
&=
\big\{
	\opt\in\optset
	\colon
	\optset-\opt\notin\textstyle\bigcap\{\rejectset[\desirset]\colon\desirset\in\setofdesirsets\}
\big\}
=
\{
	\opt\in\optset
	\colon
	\optset-\opt\notin\rejectset
\},
\end{align*}
where the first equality follows from Equation~\eqref{eq:Cfromorders}, the second from Proposition~\ref{prop:orderequivD}, the fifth from Equation~\eqref{eq:infinite:desirset:to:rejectset} and the last from our choice of $\rejectset$. Since $\rejectset$ is proper, it therefore follows from Proposition~\ref{prop:CequivK} that $\choicefun$ is proper and that $\rejectset=\rejectset[\choicefun]$. Since $\rejectset[\choicefun]=\rejectset$ is $\mathrm{mon}(\mathcal{R})$-compatible, it furthermore follows from Definition~\ref{def:properchoicefunction} that $\choicefun$ is $\mathrm{mon}(\mathcal{R})$-compatible.

For the `only if'-part of the statement, we consider any proper $\mathrm{mon}(\mathcal{R})$-compatible choice function $\choicefun$. It then follows from Proposition~\ref{prop:CequivK} and Definition~\ref{def:properchoicefunction} that $\rejectset[\choicefun]$ is a proper $\mathrm{mon}(\mathcal{R})$-compatible set of option sets. Because of Theorem~\ref{theo:infinite:abstractrepresentation:twosided:monotonification}, this implies that there is a non-empty set $\setofdesirsets$ of proper $\mathcal{R}$-compatible sets of options such that $\rejectset[\choicefun]=\bigcap\{\rejectset[\desirset]\colon\desirset\in\setofdesirsets\}$. With any $\desirset$ in $\setofdesirsets$, we now associate a binary relation $\succ_{\desirset}$ on $\opts$, defined by
\begin{equation}\label{eq:succfromD}
\opt\succ_{\desirset}\altopt
\Leftrightarrow
\opt-\altopt\in\desirset
\text{ for all $\opt,\altopt\in\opts$.}
\end{equation}
Since $\desirset$ is proper, it follows from Proposition~\ref{prop:orderequivD} that $\succ_{\desirset}$ is a proper order and that $\desirset[\succ_{\desirset}]=\desirset$. Since $\desirset$ is proper and $\mathcal{R}$-compatible, it therefore follows from Definition~\ref{def:abstractaxiom:D} that $\succ_{\desirset}$ is $\mathcal{R}$-compatible. Hence, if we let $\mathcal{O}\coloneqq\{\succ_{\desirset}\colon\desirset\in\setofdesirsets\}$, then $\mathcal{O}$ is a set of $\mathcal{R}$-compatible proper orders---$\mathcal{O}\subseteq\mathbf{O}_{\mathcal{R}}$---and, since $\mathcal{D}$ is non-empty, $\mathcal{O}$ is non-empty as well. That $\choicefun=\choicefun[\mathcal{O}]$ follows from the fact that, for all $\optset\in\optsets$:
\begin{align*}
\choicefun(\optset)
&=
\{
	\opt\in\optset
	\colon
	\optset-\opt\notin\rejectset[\choicefun]
\}\\
&=
\big\{
	\opt\in\optset
	\colon
	\optset-\opt\notin\textstyle\bigcap\{\rejectset[\desirset]\colon\desirset\in\setofdesirsets\}
\big\}\\
&=\{\opt\in\optset\colon(\exists\desirset\in\setofdesirsets)\,\optset-\opt\notin\rejectset[\desirset]\}\\
&=\{\opt\in\optset\colon(\exists\desirset\in\setofdesirsets)\,(\optset-\opt)\cap\desirset=\emptyset\}\\
&=\{\opt\in\optset\colon(\exists\desirset\in\setofdesirsets)\,(\nexists\altopt\in\optset)~\altopt-\opt\in\desirset\}\\
&=\{\opt\in\optset\colon(\exists\desirset\in\setofdesirsets)\,(\nexists\altopt\in\optset)~\altopt\succ_{\desirset}\opt\}\\
&=\{\opt\in\optset\colon(\exists \succ\,\in\raisebox{-0.7pt}{$\mathcal{O}$})\,(\nexists\altopt\in\optset)~\altopt\succ\opt\}
=
\choicefun[\mathcal{O}](\optset),
\end{align*}
where the first equality follows from Proposition~\ref{prop:CequivK}, the fourth from Equation~\eqref{eq:infinite:desirset:to:rejectset}, the sixth from Equation~\eqref{eq:succfromD} and the last from Equation~\eqref{eq:Cfromorders}.

For the final part of the statement, we consider any non-empty set $\mathcal{O}\subseteq\mathbf{O}_{\mathcal{R}}
$ such that $\choicefun=\choicefun[\mathcal{O}]$. Then for any $\succ\in\mathcal{O}$ and $\optset\in\optsets$, it follows from Equation~\eqref{eq:Cfromorders} that $\choicefun[\succ](\optset)\subseteq\choicefun[\mathcal{O}](\optset)=\choicefun(\optset)$. Since $\mathcal{O}\subseteq\mathbf{O}_{\mathcal{R}}$, this implies that $\mathcal{O}\subseteq\mathbf{O}_{\mathcal{R}}(\choicefun)$ and therefore, since $\mathcal{O}$ is non-empty, also that $\mathbf{O}_{\mathcal{R}}(\choicefun)$ is non-empty. Furthermore, for any $\optset\in\optsets$, we find that
\begin{equation*}
\choicefun(\optset)=\choicefun[\mathcal{O}](\optset)\subseteq
\choicefun[{\mathbf{O}_{\mathcal{R}}(\choicefun)}](\optset)=\bigcup_{\succ\in{\mathbf{O}_{\mathcal{R}}(\choicefun)}}\choicefun[\succ](\optset)\subseteq\choicefun(\optset),
\end{equation*}
where the first set-inclusion follows from the fact that $\mathcal{O}\subseteq\mathbf{O}_{\mathcal{R}}(\choicefun)$, the second equality follows from Equation~\eqref{eq:Cfromorders} and the last set-inclusion follows from the definition of $\mathbf{O}_{\mathcal{R}}(\choicefun)$. Hence $\choicefun=\choicefun[{\mathbf{O}_{\mathcal{R}}(\choicefun)}]$. It follows that ${\mathbf{O}_{\mathcal{R}}(\choicefun)}$ is indeed the largest non-empty set $\mathcal{O}$ of $\mathcal{R}$-compatible proper orders such that $\choicefun=\choicefun[\mathcal{O}]$.
\end{proof}

\begin{proof}{\textit{\textbf{of Theorem~\ref{theo:properrepresentationone} }}}
For the `if'-part of the statement, we consider a choice function $\choicefun$ and an $\mathcal{R}$-compatible proper order $\succ$ such that $\choicefun=\choicefun[\succ]$. Then for any $\opt,\altopt\in\opts$:
\begin{align*}
\opt\notin\choicefun(\{\opt,\altopt\})
\Leftrightarrow
\opt\notin\choicefun[\succ](\{\opt,\altopt\})
&\Leftrightarrow
(\exists w\in\{\opt,\altopt\})\,w\succ\opt
\\
&\Leftrightarrow
\opt\succ\opt\text{ or }\altopt\succ\opt
\Leftrightarrow
\altopt\succ\opt,
\end{align*}
using Equation~\eqref{eq:Cfromorder} for the second equivalence and~\ref{ax:order:irreflexive} for the last one. The order $\succ$ is therefore characterised by Equation~\eqref{eq:properrepresentationone} and hence uniquely determined by $C$. Since it follows from Equation~\eqref{eq:Cfromorder} and~\eqref{eq:Cfromorders} that $\choicefun[\succ]=\choicefun[\{\succ\}]$, we furthermore know from Theorem~\ref{theo:properrepresentation} that $\choicefun=\choicefun[\succ]=\choicefun[\{\succ\}]$ is proper and $\mathrm{mon}(\mathcal{R})$-compatible. Since $\mathcal{R}\subseteq\mathrm{mon}(\mathcal{R})$, $\choicefun$ is therefore also $\mathcal{R}$-compatible. That $\choicefun$ is binary follows from the fact that, for all $\optset\in\optsets$ and $\opt\in\optset$:
\begin{align*}
\opt\in\choicefun(\optset)
\Leftrightarrow
\opt\in\choicefun[\succ](\optset)
&\Leftrightarrow
(\nexists\altopt\in\optset)~\altopt\succ\opt\\
&\Leftrightarrow
(\nexists\altopt\in\optset)~\opt\notin\choicefun(\{\opt,\altopt\})\\
&\Leftrightarrow
(\forall\altopt\in\optset)~\opt\in\choicefun(\{\opt,\altopt\})\\
&\Leftrightarrow
(\forall\altopt\in\optset\setminus\{\opt\})~\opt\in\choicefun(\{\opt,\altopt\}),
\end{align*}
using Equation~\eqref{eq:Cfromorder} for the second equivalence, Equation~\eqref{eq:properrepresentationone} for the third equivalence, and~\ref{ax:choice:singleton} for the last one.

For the `only if'-part of the statement, we consider a choice function $\choicefun$ that is proper, binary and $\mathcal{R}$-compatible. Let $\succ$ be the binary relation that is defined by Equation~\eqref{eq:properrepresentationone}. For all $\optset\in\optsets$ and $\opt\in\optset$, we then have that
\begin{align*}
\opt\in\choicefun(\optset)
&\Leftrightarrow
(\forall\altopt\in\optset\setminus\{\opt\})~\opt\in\choicefun(\{\opt,\altopt\})\\
&\Leftrightarrow
(\forall\altopt\in\optset)~\opt\in\choicefun(\{\opt,\altopt\})\\
&\Leftrightarrow
(\nexists\altopt\in\optset)~\opt\notin\choicefun(\{\opt,\altopt\})
\Leftrightarrow
(\nexists\altopt\in\optset)~\altopt\succ\opt
\Leftrightarrow
\opt\in\choicefun[\succ](\optset),
\end{align*}
where the first equivalence follows from the binarity of $\choicefun$, the second follows from~\ref{ax:choice:singleton}, the fourth follows from Equation~\eqref{eq:properrepresentationone} and the last from Equation~\eqref{eq:Cfromorder}. Hence, $\choicefun=\choicefun[\succ]$. It remains to show that $\rejectset[\choicefun]=\rejectset[{\desirset[\succ]}]$ and that $\succ$ is an $\mathcal{R}$-compatible proper order. For any $\optset\in\optsets$, we have that
\begin{align*}
\optset\in\rejectset[\choicefun]
\Leftrightarrow
0\notin\choicefun(\optset\cup\{0\})
&\Leftrightarrow
\big(\exists\altopt\in(\optset\cup\{0\})\setminus\{0\}\big)
~0\notin\choicefun(\{0,\altopt\})\\
&\Leftrightarrow
(\exists\altopt\in\optset\setminus\{0\})
~0\notin\choicefun(\{0,\altopt\})\\
&\Leftrightarrow
(\exists\altopt\in\optset)
~0\notin\choicefun(\{0,\altopt\})\\
&\Leftrightarrow
(\exists\altopt\in\optset)
~\altopt\succ0\\
&\Leftrightarrow
(\exists\altopt\in\optset)
~\altopt\in\desirset[\succ]\\
&\Leftrightarrow
\optset\cap\desirset[\succ]\neq\emptyset
\Leftrightarrow
\optset\in\rejectset[{\desirset[\succ]}],
\end{align*}
where the first equivalence follows from Equation~\eqref{eq:fromCtoK}, the second from the binarity of $\choicefun$, the fourth from~\ref{ax:choice:singleton}, the fifth from Equation~\eqref{eq:properrepresentationone}, the sixth from Equation~\eqref{eq:fromOtoD} and the last from Equation~\eqref{eq:infinite:desirset:to:rejectset}. Hence, $\rejectset[\choicefun]=\rejectset[{\desirset[\succ]}]$. Since $\choicefun$ is proper and $\mathcal{R}$-compatible, it follows from Proposition~\ref{prop:CequivK} and Definition~\ref{def:abstractaxiom:K} that $\rejectset[\choicefun]$ is a proper set of option sets that is $\mathcal{R}$-compatible. Since $\rejectset[\choicefun]=\rejectset[{\desirset[\succ]}]$, it therefore follows from Proposition~\ref{prop:DcoherentiffKDis} that $\desirset[\succ]$ is proper and from Definitions~\ref{def:abstractaxiom:D} and~\ref{def:abstractaxiom:K} and Equation~\eqref{eq:infinite:desirset:to:rejectset} that $\desirset[\succ]$ is $\mathcal{R}$-compatible. Consider now any $\opt,\altopt\in\opts$. Then
\begin{align*}
\opt\succ\altopt
\Leftrightarrow
\altopt\notin\choicefun(\{\altopt,\opt\})
\Leftrightarrow
0\notin\choicefun(\{0,\opt-\altopt\})
\Leftrightarrow
\opt-\altopt\succ0
\Leftrightarrow
\opt-\altopt\in\desirset[\succ],
\end{align*}
where the first and third equivalence follow from Equation~\eqref{eq:properrepresentationone}, the second follows from~\ref{ax:choice:translation} and the fourth follows from Equation~\eqref{eq:fromOtoD}. Since $\desirset[\succ]$ is a proper set of options, it therefore follows from Proposition~\ref{prop:orderequivD} that $\succ$ is a proper order. That $\succ$ is $\mathcal{R}$-compatible, finally, follows from Definition~\ref{def:abstractaxiom:D} and the properness and $\mathcal{R}$-compatibility of $\desirset[\succ]$.
\end{proof}

}{}

\end{document}